\documentclass{article}

\usepackage[final,nonatbib]{neurips_2019}

\usepackage[T1]{fontenc}
\usepackage{inconsolata}
\usepackage{mathptmx}

\usepackage{hyperref}       
\usepackage{url}            
\usepackage{booktabs}       
\usepackage{amsfonts}       
\usepackage{nicefrac}       
\usepackage{microtype}      

\usepackage[english]{babel}
\usepackage{amssymb,amsmath,amsthm,nccmath}
\usepackage{mathtools}

\usepackage[dvipsnames]{xcolor}
\usepackage{hyperref}

\usepackage{float}
\usepackage{enumerate}
\usepackage{dsfont}

\usepackage{accents}

\usepackage{bigstrut}
\usepackage{multirow}
\usepackage{scalerel}

\usepackage{algorithm2e}

\usepackage[round, sort]{natbib}

\def\L{\mathcal{L}}
\def\J{\mathcal{J}}
\def\nust{\nu^{\ast}}

\def\H{\mathcal{H}}
\def\mcT{\mathcal{T}}

\def\mfN{\mathfrak{N}}

\def\mcZ{\mathcal{Z}}

\newcommand{\mathLarger}[1]{\mathlarger{\mathlarger{#1}}}

\def\dt{\tilde{d}}
\def\tA{\tilde{A}}
\def\tL{\tilde{L}}

\def\A{\mathcal{A}}

\def\F{{\mathcal{F}}}
\def\G{{\mathcal{G}}}

\def\M{\mathcal{M}}

\def\mcT{\mathcal{T}}

\def\yt{\hat{y}}
\def\Pt{\bar{P}}
\def\Bt{\hat{B}}

\def\mfN{{\mathfrak{N}}}

\def\mcM{{\mathcal{M}}}

\def\mcE{\mathcal{E}}

\def\xst{x^{\star}}

\def\tM{\tilde{\mathcal{M}}}
\def\tMz{\tilde{\mathcal{M}}^{\scaleto{\mathstrut(0)}{4.5pt}}}

\def\H{\mathcal{H}}

\def\E{\mathbb{E}}
\def\P{\mathbb{P}}
\def\P{{\mathbb{P}}}
\def\R{\mathbb{R}}

\def\tpz{\tilde{p}^{\scaleto{\mathstrut(0)}{4.5pt}}}
\def\tPhi{\tilde{\Phi}}

\def\tomega{\tilde{\omega}}

\def\T{\intercal}

\def\hst{h^{\hspace{0.1em}\mathclap{\ast}}}

\renewcommand{\hbar}{\overline{h}}

\include{nocommenting}

\newtheorem{theorem}{Theorem}[section]
\newtheorem{proposition}{Proposition}[section]
\newtheorem{lemma}{Lemma}[section]

\newtheorem{remark}{Remark}

\title{A Latent Variational Framework for Stochastic Optimization}

\author{%
  Philippe~Casgrain \\
  Department of Statistical Sciences\\
  University of Toronto\\
  Toronto, ON, Canada\\
  \texttt{p.casgrain@mail.utoronto.ca}
}

\begin{document}

\maketitle

\begin{abstract}
	This paper provides a unifying theoretical framework for stochastic optimization algorithms by means of a latent stochastic variational problem. Using techniques from stochastic control, the solution to the variational problem is shown to be equivalent to that of a Forward Backward Stochastic  Differential Equation (FBSDE). By solving these equations, we recover a variety of existing adaptive stochastic gradient descent methods. This framework establishes a direct connection between stochastic optimization algorithms and a secondary latent inference problem on gradients, where a prior measure on gradient observations determines the resulting algorithm.
\end{abstract}


\section{Introduction}
\label{sec:Introduction}

Stochastic optimization algorithms are tools which are crucial to solving optimization problems arising in machine learning. 
The initial motivation for these algorithms comes from the fact that computing the gradients of a target loss function becomes increasingly difficult as the scale and dimension of an optimization problem grows larger.
In these large-scale optimization problems, deterministic gradient-based optimization algorithms perform poorly due to the computational load of repeatedly computing gradients. Stochastic optimization algorithms remedy this issue by replacing exact gradients of the target loss with a computationally cheap gradient estimator, trading off noise in gradient estimates for computational efficiency at each step.

To illustrate this idea, consider the problem of minimizing a generic risk function $f:\R^d \rightarrow \R$, taking the form
\begin{equation} \label{eq:risk-function-motivation}
	f(x) =  \frac{1}{|\mfN|} \sum_{z\in\mfN} \ell(x;z)
	\;,
\end{equation}
where $\ell:\R^d \times \mcZ \rightarrow \R$, and where we define the set $\mfN := \{ z_i \in \mcZ \;,\; i=1,\dots, N \}$ to be a set of training points. In this definition, we interpret $\ell(x;z)$ as the model loss at a single training point $z\in\mfN$ for the parameters $x\in\R^d$. 

When $N$ and $d$ are typically large, computing the gradients of $f$ can be time-consuming. Knowing this, let us consider the path of an optimization algorithm as given by $\{x_t\}_{t\in\mathbb{N}}$. Rather than computing $\nabla f (x_t)$ directly at each point of the optimization process, we may instead collect noisy samples of gradients as
\begin{equation} \label{eq:risk-gradient-sample-motivation}
	g_t =  \frac{1}{|\mfN_t^m|} \sum_{z\in\mfN_t^m} \nabla_x \ell(x_t;z)
	\;,
\end{equation}
where for each $t$, $\mfN_t^m \subseteq \mfN$ is an independent sample of size $m$ from the set of training points. We assume that $m \ll N$ is chosen small enough so that $g_t$ can be computed at a significantly lower cost than $\nabla f(x_t)$. Using the collection of noisy gradients $\{g_t\}_{t\in\mathbb{N}}$, stochastic optimization algorithms construct an estimator $\widehat{\nabla f}(x_t)$ of the gradient $\nabla f(x_t)$ in order to determine the next step $x_{t+1}$ of the optimizer.

This paper presents a theoretical framework which provides new perspectives on stochastic optimization algorithms, and explores the implicit model assumptions that are made by existing ones. We achieve this by extending the 
 approach taken by~\citet{wibisono2016variational} to stochastic algorithms. The key step in our approach is to interpret the task of optimization with a stochastic algorithm as a latent variational problem. As a result, we can recover algorithms from this framework which have built-in online learning properties. In particular, these algorithms use an online Bayesian filter on the stream of noisy gradient samples, $g_t$, to compute estimates of $\nabla f(x_t)$. Under various model assumptions on $\nabla f$ and $g$, we recover a number of common stochastic optimization algorithms. 


\subsection{Related Work}

There is a rich literature on stochastic optimization algorithms as a consequence of their effectiveness in machine learning applications. Each algorithm introduces its own variation on the gradient estimator ${\widehat{\nabla f}(x_t)}$ as well as other features which can improve the speed of convergence to an optimum. Amongst the simplest of these is \emph{stochastic gradient descent} and its variants~\citet{robbins1951stochastic}, which use an estimator based on single gradient samples. Others, such as~\citet{lucas2018aggregated,nesterov27method}, use momentum and acceleration as features to enhance convergence, and can be interpreted as using exponentially weighted moving averages as gradient estimators. Adaptive gradient descent methods such as AdaGrad from~\citet{duchi2011adaptive} and Adam from~\citet{kingma2014adam} use similar moving average estimators, as well as dynamically updated normalization factors. For a survey paper which covers many modern stochastic optimization methods, see~\citet{ruder2016overview}.

There exist a number of theoretical interpretations of various aspects of stochastic optimization. \cite{cesa2004generalization} have shown a parallel between stochastic optimization and online learning. Some previous related works, such as~\citet{gupta2017unified} provide a general model for adaptive methods, generalizing the subgradient projection approach of~\cite{duchi2011adaptive}. 
\citet{aitchison2018unified} use a Bayesian model to explain the various features of gradient estimators used in stochastic optimization algorithms . This paper differs from these works by naturally generating stochastic algorithms from a variational principle, rather than attempting to explain their individual features. This work is most similar to that of \citet{wibisono2016variational} who provide a variational model for continuous deterministic optimization algorithms.

There is a large body of research on continuous-time approximations to deterministic optimization algorithms via dynamical systems (ODEs) (\cite{su2014differential,krichene2015accelerated,wilson2016lyapunov,da2018general}), as well as approximations to stochastic optimization algorithms by stochastic differential equations (SDEs) (\cite{xu2018accelerated,xu2018continuous,raginsky2012continuous,mertikopoulos2018convergence,krichene2017acceleration}). In particular, the most similar of these works,~\citet{raginsky2012continuous,xu2018accelerated,xu2018continuous}, study continuous approximations to stochastic mirror descent by adding exogenous Brownian noise to the continuous dynamics derived in~\citet{wibisono2016variational}. This work differs by deriving continuous stochastic dynamics for optimizers from a broader theoretical framework, rather than positing the continuous dynamics as-is. Although the equations studied in these papers may resemble some of the results derived in this one, they differ in a number of ways. Firstly, this paper finds that the source of randomness present in the optimizer dynamics obtained in this paper are not generated by an exogenous source of noise, but are in fact an explicit function of the randomness generated by observed stochastic gradients during the optimization process. Another important difference is that the optimizer dynamics presented in this paper make no use of the gradients of the objective function, $\nabla f$ (which is inaccessible to a stochastic optimizer), and are only a function of the stream of stochastic gradients $g_t$.

\subsection{Contribution}

To the author's knowledge, this is the first paper to produce a theoretical model for stochastic optimization based on a variational interpretation. This paper extends the continuous variational framework~\cite{wibisono2016variational} to model stochastic optimization. From this model, we derive optimality conditions in the form of a system of forward-backward stochastic differential equations (FBSDEs), and provide bounds on the expected rate of convergence of the resulting optimization algorithm to the optimum. By discretizing solutions of the continuous system of equations, we can recover a number of well-known stochastic optimization algorithms, demonstrating that these algorithms can be obtained as solutions of the variational model under various assumptions on the loss function, $f(x)$, that is being minimized.

\subsection{Paper Structure}

In Section~\ref{sec:The-Optimization-Model} we define a continuous-time surrogate model of stochastic optimization. Section~\ref{sec:optimizer-variational-problem} uses this model to motivate a stochastic variational problem over optimizers, in which we search for stochastic optimization algorithms which achieve optimal average performance over a collection of minimization problems. In Section~\ref{sec:Critical-Points-Action} we show that the necessary and sufficient conditions for optimality of the variational problem can be expressed as a system of Forward-Backward Stochastic Differential Equations. Theorem~\ref{thm:convergence-rate} provides rates of convergence for the optimal algorithm to the optimum of the minimization problem. Lastly, Section~\ref{sec:Recovering-Algorithms} recovers SGD, mirror descent, momentum, and other optimization algorithms as discretizations of the continuous optimality equations derived in Section~\ref{sec:Critical-Points-Action} under various model assumptions. The proofs of the mathematical results of this paper are found within the appendices.

\section{A Statistical Model for Stochastic Optimization}
\label{sec:The-Optimization-Model}

Over the course of the section, we present a variational model for stochastic optimization. The ultimate objective will be to construct a framework for measuring the average performance of an algorithm over a random collection of optimization problems.
We define random variables in an ambient probability space $\smash{(\Omega,\P,\mathfrak{G}=\{\G_t\}_{t \in [0,T]})}$, where $\G_t$ is a filtration which we will define at a later point in this section. 
We assume that loss functions are drawn from a random variable $f:\Omega\rightarrow C^1(\R^d)$. Each draw from the random variable satisfies $f(x)\in\R$ for fixed $x\in\R^d$, and $f$ is assumed to be an almost-surely continuously differentiable in $x$. In addition, we make the technical assumption that $\E\, \lVert \nabla f(x) \rVert^2 < \infty$ for all $x\in\R^d$.

We define an optimizer $X=(X_t^\nu)_{t\geq 0}$ as a controlled process satisfying $X_t^{\nu}\in\R^d$ for all $t\geq 0$, with initial condition $X_0\in\R^d$. The paths of $X$ are assumed to be continuously differentiable in time so that the dynamics of the optimizer may be written as $d X_t^{\nu} = \nu_t \,dt$, where $\nu_t \in \R^d$ represents the control, where we use the superscript to express the explicit dependence of $X^{\nu}$ on the control $\nu$. We may also write the optimizer in its integral form as $X_t^{\nu} = X_0 + \int_0^t \nu_u \, du$, demonstrating that the optimizer is entirely characterized by a pair $(\nu,X_0)$ consisting of a control process $\nu$ and an initial condition $X_0$. Using an explicit Euler discretization with step size $\epsilon>0$, the optimizer can be approximately represented through the update rule $X_{t+\epsilon}^{\nu} \approx X^{\nu}_t + \epsilon \, \nu_t$. This leads to the interpretation of $\nu_t$ as the (infinitesimal) step the algorithm takes at each point $t$ during the optimization process.

In order to capture the essence of stochastic optimization, we construct our model so that optimizers have restricted access to the gradients of the loss function $f$. Rather than being able to directly observe $\nabla f$ over the path of $X_t^{\nu}$, we assume that the algorithm may only use a noisy source of gradient samples, modeled by a {c\`adl\`ag semi-martingale}\footnote{A \emph{c\`adl\`ag} (continue \`a droite, limite \`a gauche) process is a continuous time process that is almost-surely right-continuous with finite left limit at each point t. A \emph{semi-martingale} is the sum of a process of finite variation and a local martingale. For more information on continuous time stochastic processes and these definitions, see the canonical text~\citet{jacod2013limit}.}
$g=\left( g_t \right)_{t \geq 0}$. As a simple motivating example, we can consider the model $g_t = \nabla f(X_t^{\nu}) + \xi_t$, where $\xi_t$ is a white noise process. This particular model for the noisy gradient process can be interpreted as consisting of observing $\nabla f(X_t^{\nu})$ plus an independent source of noise. This concrete example will be useful to keep in mind to make sense of the results which we present over the course of the paper.

To make the concept of information restriction mathematically rigorous, we restrict ourselves only to optimizers $X^{\nu}$ which are measurable with respect to the information generated by the noisy gradient process $g$. To do this, we first define the global filtration $\G$, as $\G_t = \sigma\left( (g_u)_{u\in[0,t]} , f \right)$ as the sigma algebra generated by the paths of $g$ as well as the realizations of the loss surface $f$. The filtration $\G_t$ is defined so that it contains the complete set of information generating the optimization problem until time $t$.

Next, we define the coarser filtration $\F_t = \sigma(g_u)_{u\in[0,t]}\subset \G_t$ generated strictly by the paths of the noisy gradient process. This filtration represents the total set of information \emph{available to the optimizer} up until time $t$. This allows us to formally restrict the flow of information to the algorithm by restricting ourselves to optimizers which are adapted to $\F_t$. More precisely, we say that the optimizer's control $\nu$ is admissible if \useshortskip
\begin{equation}
	\nu \in \A := \left\{ 
	\omega = \left( \omega_t \right)_{t \geq 0}
	\, : \;
	\omega \text{ is $\F$-adapted}
	\, , \;
	\E\int_0^T \, \lVert \omega_t \lVert^2 + \lVert \nabla f( X^{\omega}_t ) \lVert^2 \, dt < \infty
	\right\}
	\;.
\end{equation}
The set of optimizers generated by $\A$ can be interpreted as the set of optimizers which may only use the source of noisy gradients, which have bounded expected travel distance and have square-integrable gradients over their path.

\section{The Optimizer's Variational Problem}
\label{sec:optimizer-variational-problem}

Having defined the set of admissible optimization algorithms, we set out to select those which are optimal in an appropriate sense. We proceed similarly to~\citet{wibisono2016variational}, by proposing an objective functional which measures the performance of the optimizer over a finite time period. 

The motivation for the optimizer's performance metric comes from a physical interpretation of the optimization process. We can think of our optimization process as a particle traveling through a potential field define by the target loss function $f$. As the particle travels through the potential field, it may either gain or lose momentum depending on its location and velocity, which will in turn affect the particle's trajectory. Naturally, we may seek to find the path of a particle which reaches the optimum of the loss function while minimizing the total amount of kinetic and potential energy that is spent. We therefore turn to the Lagrangian interpretation of classical mechanics, which provides a framework for obtaining solutions to this problem. Over the remainder of this section, we lay out the Lagrangian formalism for the optimization problem we defined in Section~\ref{sec:The-Optimization-Model}.

To define a notion of energy in the optimization process, we provide a measure of distance in the parameter space. We use the \emph{Bregman Divergence} as the measure of distance within our parameter space, which can embed additional information about the geometry of the optimization problem.
The Bregman divergence, $D_h$, is defined as
\begin{equation} \label{eq:Bregman-Divergence}
	D_h(y,x) = h(y) - h(x) - \langle \nabla h(x), y-x \rangle
\end{equation}
where $h:\R^d \rightarrow \R$ is a strictly convex function satisfying $h\in C^2$. We assume here that the gradients of $h$ are $L$-Lipschitz smooth for a fixed constant $L>0$. The choice of $h$ determines the way we measure distance, and is typically chosen so that it mimics features of the loss function $f$. In particular, this quantity plays a central role in mirror descent and non-linear sub-gradient algorithms. For more information on this connection and on Bregman Divergence, see~\citet{nemirovsky1983problem} and~\citet{beck2003mirror}.

We define the total energy in our problem as the kinetic energy, accumulated through the movement of the optimizer, and the potential energy generated by the loss function $f$. Under the assumption that $f$ almost surely admits a global minimum $\xst=\arg\min_{x\in\R^d} f(x)$, we may represent the total energy via the Bregman Lagrangian as
\begin{equation} \label{eq:Bregman-Lagrangian}
	\L(t,X,\nu) = 
	e^{\gamma_t} 
	{ (}
	\underbrace{ \vphantom{\int} e^{\alpha_t} D_h\left( X + e^{-\alpha_t} \nu , X \right)}_{\mathclap{\text{Kinetic Energy}}} - 
	\underbrace{ \vphantom{\int} e^{\beta_t} \left(f(X) - f(\xst)\right) }_{\mathclap{\text{Potential Energy}}} 
	{ )}
	\;,
\end{equation}
for fixed inputs $(t,X,\nu)$, and where we assume that $\gamma,\alpha,\beta: \R^+ \rightarrow \R$ are deterministic, and satisfy $\gamma,\alpha,\beta \in C^1$. The functions $\gamma,\alpha,\beta$ can be interpreted as hyperparameters which tune the energy present at any state of the optimization process. An important property to note is that the Lagrangian is itself a random variable due to the randomness introduced by the latent loss function $f$.

The objective is then to find an optimizer within the admissible set $\A$ which can get close to the minimum $\xst=\min_{x\in\R^d} f(x)$, while simultaneously minimizing the energy cost over a finite time period $[0,T]$. The approach taken in classical mechanics and in~\citet{wibisono2016variational} fixes the endpoint of the optimizer at $\xst$. Since we assume that the function $f$ is not directly visible to our optimizer, it is not possible to add a constraint of this type that will hold almost surely. Instead, we introduce a soft constraint which penalizes the algorithm's endpoint in proportion to its distance to the global minimum, $f(X_T)-f(\xst)$. As such, we define the \emph{expected action functional} $\J:\A \rightarrow \R$ as
\begin{equation} \label{eq:Expected-Action-Functional}
	\J(\nu) = 
	\E {\Big [}
	\;
	\underbrace{
	\int_0^T 
	\L(t,X_t^\nu,\nu_t)
	\, dt
	}_{\text{Total Path Energy}}
	+
	\underbrace{ \vphantom{\int_0}
	e^{\delta_T} \left( \vphantom{\sum} f(X_T^{\nu}) - f(\xst)\right)
	}_{\text{Soft End Point Constraint}}
	{\Big ]}
	\;,
\end{equation}
where $\delta_T \in C^1$ is assumed to be an additional model hyperparameter, which controls the strength of the soft constraint.

With this definition in place, the objective will be to select amongst admissible optimizers for those which minimize the expected action. Hence, we seek optimizers which solve the stochastic variational problem
\begin{equation} \label{eq:Optimization-Problem}
	\nust = \arg\min_{\nu\in\A} \J(\nu)
	\;.
\end{equation}

\begin{remark}
	Note that the variational problem~\eqref{eq:Optimization-Problem} is identical to the one with Lagrangian
	\begin{equation}
	\tilde{\L}(t,X,\nu) = 
	e^{\gamma_t} 
	(e^{\alpha_t} D_h\left( X + e^{-\alpha_t} \nu , X \right) - 
	e^{\beta_t} f(X) 
	)
	\end{equation}
	and terminal penalty $e^{\delta_T} f(X_T^{\nu})$, since they differ by constants independent of $\nu$. Because of this, the results presented in Section~\ref{sec:Critical-Points-Action} also hold the case where $\xst$ and $f(\xst)$ do not exist or are infinite.
\end{remark}


\section{Critical Points of the Expected Action Functional}
\label{sec:Critical-Points-Action}

In order to solve the variational problem~\eqref{eq:Optimization-Problem}, we make use techniques from the calculus of variations and infinite dimensional convex analysis to provide optimality conditions for the variational problem~\eqref{eq:Optimization-Problem}. To address issues of information restriction, we rely on the stochastic control techniques developed by~\citet{casgrain2018algorithmic,casgrain2018trading,casgrain2018mean}.

The approach we take relies on the fact that a necessary condition for the optimality of a G\^ateaux differentiable functional $\J$ is that its G\^ateaux derivative vanishes in all directions. Computing the G\^ateaux derivative of $\J$, we find an equivalence between the G\^ateaux derivative vanishing and a system of Forward-Backward Stochastic Differential Equations (FBSDEs), yielding a generalization of the Euler-Lagrange equations to the context of our optimization problem. The precise result is stated in Theorem~\ref{thm:Stochastic-Euler-Lagrange} below.

\begin{theorem}[Stochastic Euler-Lagrange Equation]
\label{thm:Stochastic-Euler-Lagrange}
A control $\nust\in\A$ is a critical point of $\J$ if and only if $((\frac{\partial \L}{\partial \nu}), \mcM)$ is a solution to the system of FBSDEs,
\begin{equation} \label{eq:Thm-Optimality-BSDE}
		d \left(\frac{\partial \L}{\partial \nu}\right)_t = \E\left[ \left(\frac{\partial \L}{\partial X}\right)_t {\Big \lvert} \F_t \right] \, dt + d\mcM_t
		\;\; \forall t<T \,,\;\;
		\left(\frac{\partial \L}{\partial \nu}\right)_{T} = 
		-e^{\delta_T} \, \E\left[ \nabla f(X_T) {\Big \lvert} \F_T  \right]
		\;,
\end{equation}
where we define the processes
\begin{align}
	\label{eq:BSDE-dLdX}
	{
	\left(\frac{\partial \L}{\partial X}\right)_t }
	&=
	e^{\gamma_t+\alpha_t} \mathLarger{(} \,
	\nabla h(X_t^{\nust} + e^{-\alpha_t} \nust_t )  - \nabla h(X_t^{\nust})
	-  e^{-\alpha_t} \nabla^2 h(X_t^{\nust}) \nust_t  - e^{\beta_t} \nabla f(X_t^{\nust})
	\, \mathLarger{)} 
	\\
	\vphantom{\mathLarger{{\int}}}
	\left(\frac{\partial \L}{\partial \nu}\right)_t
	&=
	e^{\gamma_t} \left(
	\nabla h(X_t^{\nust} + e^{-\alpha_t} \nust_t ) - \nabla h(X_t^{\nust})
	\right)
	\;, \label{eq:BSDE-dLdnu}
\end{align} 
and where the process $\M=(\M_t)_{t\in[0,T]}$ is an $\F$-adapted martingale. As a consequence, if the solution to this FBSDE is unique, then it is the unique critical point of the functional $\J$ up to null sets.
\end{theorem}
\vspace{-1em}
\begin{proof}
	See Appendix~\ref{sec:proof-thm-Stochastic-Euler-Lagrange}
\end{proof}

Theorem~\ref{thm:Stochastic-Euler-Lagrange} presents an analogue of the Euler-Lagrange equation with free terminal boundary. Rather than obtaining an ODE as in the classical result, we obtain an FBSDE\footnote{For a background on FBSDEs, we point readers to \citet{pardoux1999forward,ma1999forward,carmona2016lectures}. At a high level, the solution to an FBSDE of the form~\eqref{eq:Thm-Optimality-BSDE} consists of a pair of processes $(\nicefrac{\partial \L}{\partial \nu},\M)$, which simultaneously satisfy the dynamics and the boundary condition of~\eqref{eq:Thm-Optimality-BSDE}. Intuitively, the martingale part of the solution can be interpreted as a random process which guides $(\nicefrac{\partial \L}{\partial X})_t$ towards the boundary condition at time $T$.
}, with backwards process $(\nicefrac{\partial \L}{\partial \nu})_t$, and forward state processes $\E[ (\nicefrac{\partial \L}{\partial X})_t \lvert \F_t ]$, $\int_0^t \left\| \nu_u \right\| \, du$ and $X_t^{\nust}$. We can also interpret the dynamics of equation~\eqref{eq:Thm-Optimality-BSDE} as being the filtered optimal dynamics of~\cite[Equation 2.3]{wibisono2016variational}, $\E [ (\nicefrac{\partial \L}{\partial X})_t \lvert \F_t ]$, plus the increments of data-dependent martingale $\mcM_t$, with mechanics similar to that of the `innovations process' of filtering theory. This martingale term should not be interpreted as a source of noise, but as an explicit function of the data, as is evident from its explicit form 
\begin{equation}
	\mcM_t = 
	\E\left[ \int_0^T \left(\frac{\partial \L}{\partial X}\right)_u \, du 
	-e^{\delta_T} \, \nabla f(X_T)
	{\Big\lvert} \F_t\right] \;.
\end{equation}

A feature of equation~\eqref{eq:Thm-Optimality-BSDE}, is that optimality relies on the projection of $(\nicefrac{\partial \L}{\partial X})_t$ onto $\F_t$. Thus, the optimization algorithm makes use of past noisy gradient observations in order to make local gradient predictions. Local gradient predictions are updated using a Bayesian mechanism, where the prior model for $\nabla f$ is conditioned with the noisy gradient information contained in $\F_t$. This demonstrates that the solution depends only on the gradients of $f$ along the path of $X_t$ and no higher order properties.



\subsection{Expected Rates of Convergence of the Continuous Algorithm}

Using the dynamics~\eqref{eq:Thm-Optimality-BSDE} we obtain a bound on the rate of convergence of the continuous optimization algorithm that is analogous to~\citet[Theorem 2.1]{wibisono2016variational}. 
We introduce the Lyapunov energy functional \useshortskip
\begin{equation} \label{eq:Energy-Functional}
	\mcE_t = D_h(\xst,X^{\nust}_t + e^{-\alpha_t} \nu_t) + e^{\beta_t} \left( f(X^{\nust}_t) - f(\xst)\right) 
	-  [ \nabla h( X^{\nust} + e^{-\alpha_t} \nu ) , X^{\nust} + e^{-\alpha_t} \nu  ]_t
	\;,
\end{equation}
where we define $\xst$ to be a global minimum of $f$. Under additional model assumptions, and by showing that this quantity is a super-martingale with respect to the filtration $\F$, we obtain an upper bound for the expected rate of convergence from $X_t$ towards the minimum.

\begin{theorem}[Convergence Rate] \label{thm:convergence-rate}
	Assume that the function $f$ is almost surely convex and that the \emph{scaling conditions} $\dot\gamma_t = e^{\alpha_t}$ and $\dot\beta_t \leq e^{\alpha_t}$ hold. Moreover, assume that in addition to $h$ having $L$-Lipschitz smooth gradients, $h$ is also $\mu$-strongly-convex with $\mu>0$. Define $\xst = \arg \min_{x\in\R^d} f(x)$ to be a global minimum of $f$. If $\xst$ exists almost surely, the optimizer defined by FBSDE~\eqref{eq:Thm-Optimality-BSDE} satisfies
	\begin{equation} \label{eq:converence-bound}
		\E\left[ f(X_t) - f(\xst) \right] 
		= 
		O\left( e^{-\beta_t} 
		\max
		\left\{ 1 \,,
		\E\left[ \, [ e^{-\gamma_t} \mcM ]_t  \right]
		\right\}
		\right)
		\;,
	\end{equation} 
	where $\left[ e^{-\gamma_t} \M\right]_t$ represents the quadratic variation of the process $e^{-\gamma_t} \mcM_t$, where $\M$ is the martingale part of the solution defined in Theorem~\ref{thm:Stochastic-Euler-Lagrange}.
\end{theorem}
\vspace{-1.25em}
\begin{proof}
	See Appendix~\ref{sec:proof-thm-convergence-rate}.
\end{proof}
\vspace{-0.5em}


We may interpret the term $\E\left[ \, [ e^{-\gamma_t}\mcM ]_t\right]$ as a penalty on the rate of convergence, which scales with the amount of noise present in our gradient observations. To see this, note that if there is no noise in our gradient observations, we obtain that $\F_t = \G_t$, and hence $\M_t\equiv0$, which recovers the exact deterministic dynamics of \cite{wibisono2016variational} and the optimal convergence rate $O(e^{-\beta_t})$. If the noise in our gradient estimates is large, we can expect $\E\left[ \, [ e^{-\gamma}\mcM ]_t\right]$ to grow at quickly and to counteract the shrinking effects of $e^{-\beta_t}$. Thus, in the case of a convex objective function $f$, any presence of gradient noise will proportionally hurt rate of convergence to an optimum. We also point out, that there will be a nontrivial dependence of $\E\left[ \, [ e^{-\gamma}\mcM ]_t\right]$ on all model hyperparameters, the specific definition of the random variable $f$, and the model for the noisy gradient stream, $(g_t)_{t\geq 0}$.

\begin{remark}
	We do not assume that the conditions of Theorem~\ref{thm:convergence-rate} carry throughout the remainder of the paper. In particular, Sections~\ref{sec:Recovering-Algorithms} study models which may not guarantee almost-sure convexity of the latent loss function. 
\end{remark}

\section{Recovering Discrete Optimization Algorithms}
\label{sec:Recovering-Algorithms}

In this section, we use the optimality equations of Theorem~\ref{thm:Stochastic-Euler-Lagrange} to produce discrete stochastic optimization algorithms. The procedure we take is as follows. We first define a model for the processes $(\nabla f(X_t), g_t )_{t \in [0,T]}$. Second, we solve the optimality FBSDE~\eqref{eq:Thm-Optimality-BSDE} in closed form or approximate the solution via the first-order singular perturbation (FOSP) technique, as described in Appendix~\ref{sec:Solution-Techniques-for-Optimality-FBSDE}. Lastly, we discretize the solutions with a simple Forward-Euler scheme in order to recover discrete algorithms. 

Over the course of Sections~\ref{sec:mini-batch-SGD-mirror-descent} and~\ref{sec:Kalman-Algorithms-Section}, we show that various simple models for $(\nabla f(X_t), g_t )_{t \in [0,T]}$ and different specifications of $h$ produce many well-known stochastic optimization algorithms. These establish the conditions, in the context of the variational problem of Section~\ref{sec:The-Optimization-Model}, under which each of these algorithms are optimal. As a consequence, this allows us to understand the prior assumptions which these algorithms make on the gradients of the objective function they are trying to minimize, and the way noise is introduced in the sampling of stochastic gradients, $(g_t)_{t\geq 0}$.

\subsection{Stochastic Gradient Descent and Stochastic Mirror Descent}
\label{sec:mini-batch-SGD-mirror-descent}

Here we propose a Gaussian model on gradients which loosely represents the behavior of mini-batch stochastic gradient descent with a training set of size $n$ and mini-batches of size $m$. By specifying a martingale model for $\nabla f(X_t)$, we recover the stochastic gradient descent and stochastic mirror descent algorithms as solutions to the variational problem described in Section~\ref{sec:The-Optimization-Model}.

Let us assume that $\nabla f(X_t) = \sigma W_t^{f}$, where $\sigma>0$ and $(W^f_t)_{t\geq 0}$ is a Brownian motion. Next, assume that the noisy gradients samples obtained from mini-batches over the course of the optimization, evolve according to the model $\smash g_t = \sigma ( W_t^f +  \rho W_t^{e})$, where $\rho=\sqrt{\nicefrac{(n-m)}{m}}$ and $W^{e}$ is an independent copy of $W_t^f$. Here, we choose $\rho$ so that $\mathbb{V}[g_t] = (\nicefrac{n}{m}) \mathbb{V} [\nabla f(X_t)]=O(m^{-1})$, which allows the variance to scale in $m$ and $n$ as it does with mini-batches.

Using symmetry, we obtain the trivial solution to the gradient filter, $\E[ \nabla f(X_t) {\lvert} \F_t ] =  (1+\rho^2)^{-1} g_t$, implying that the best estimate of the gradient at the point $X_t$ will be the most recent mini-batch sample observed. re-scaled by a constant depending on $n$ and $m$. Using this expression for the filter, we obtain the following result.

\begin{proposition} \label{prop:FOSP-Martingale-SGD-cts-Solution}
	The FOSP approximation to the solution of the optimality equations~\eqref{eq:Thm-Optimality-BSDE} can be expressed as \useshortskip
	\begin{equation} \label{eq:FOSP-Martingale-SGD-cts-Solution}
		dX_t = e^{\alpha_t} \left( 
	\nabla h^{\ast} \left( \nabla h(X_t) - \tPhi_t (1+\rho^2)^{-1} g_t  \right) 
	- X_t^{\nust}
	\right) \,dt
	\;,
	\end{equation}
	where $h^{\ast}$ is the convex dual of $h$ and where $\tPhi_t =
		e^{-\gamma_t} ( \Phi_0 + \int_{0}^{t} e^{\alpha_u + \beta_u + \gamma_u } \, du )$ is a deterministic learning rate
	with $\smash{\Phi_0 = e^{\delta_T} - \int_0^T e^{\alpha_u + \beta_u + \gamma_u} \, du}$. When $h$ has the form $h(x)=x^{\T} M x$ for a symmetric positive-definite matrix $M$, the FOSP approximation is exact, and~\eqref{eq:FOSP-Martingale-SGD-cts-Solution} is the exact solution to the optimality FBSDE~\eqref{eq:Thm-Optimality-BSDE}. The martingale portion of the solution to~\eqref{eq:Thm-Optimality-BSDE} can be expressed as $\mcM_t = \mcM_0 -(1+\rho^2)^{-1}\int_0^t e^{\alpha_u + \beta_u + \gamma_u} \, dg_u$.
\end{proposition}
\vspace{-1.2em}
\begin{proof}
	See Appendix~\ref{sec:proof-prop-FOSP-Martingale-SGD-cts-Solution}.
\end{proof}
\vspace{-1em}

To obtain a discrete optimization algorithm from the result of~\ref{prop:FOSP-Martingale-SGD-cts-Solution}, we employ a forward-Euler discretization of the ODE~\eqref{eq:FOSP-Martingale-SGD-cts-Solution} on the finite mesh $\mcT =\{ t_0=0 \,,\; t_{k+1} = t_{k} + e^{-\alpha_{t_k}} :  k\in\mathbb{N} \}$.
This discretization results in the update rule \useshortskip
\begin{equation}
 	X_{t_{k+1}}
 	=
	\nabla h^{\ast} \left( \nabla h(X_{t_k}) - \tPhi_{t_k} \, g_{t_k} \right) 
	\;,
\end{equation} 
corresponding exactly to mirror descent (e.g.\ see~\cite{beck2003mirror}) using the noisy mini-batch gradients $g_t$ and a time-varying learning rate $\tPhi_{t_k}$. Moreover, setting $h(x) =\frac{1}{2} \|x\|^2$, we recover the update rule
$X_{t_{k+1}} - X_{t_k} = - \tPhi_{t_k} \, g_{t_k}$, exactly
corresponding to the mini-batch SGD with a time-dependent learning rate. 

This derivation demonstrates that the solution to the variational problem described in Section~\ref{sec:The-Optimization-Model}, under the assumption of a Gaussian model for the evolution of gradients, recovers mirror descent and SGD. In particular, the martingale gradient model proposed in this section can be roughly interpreted as assuming that gradients behave as random walks over the path of the optimizer. Moreover, the optimal gradient filter $\E[ \nabla f(X_t) {\lvert} \F_t ] =  (1+\rho^2)^{-1} g_t$ shows that, for the algorithm to be optimal, mini-batch gradients should be re-scaled in proportion to $(1 + \rho^2)^{-1} =\nicefrac{m}{n}$.

\subsection{Kalman Gradient Descent and Momentum Methods}
\label{sec:Kalman-Algorithms-Section}

Using a \emph{linear state-space model} for gradients, we can recover both the Kalman Gradient Descent algorithm of~\citet{vuckovic2018kalman} and momentum-based optimization methods of~\citet{polyak1964some}. 
We assume that each component of $\smash{\nabla f(X_t) = (\nabla_i f(X_t))_{i=1}^d}$ is modeled independently as a linear diffusive process. Specifically, we assume that there exist processes $\smash{y_i=(y_{i,t})_{t \geq 0}}$ so that for each $i$, $\smash{\nabla_i f(X_{t}) =  b^{\T} y_{i,t}}$, where $\smash{y_{i,t} \in \R^{\dt}}$ is the solution to the linear SDE $\smash{dy_{i,t} = - A \, y_{i,t} dt + L \, dW_{i,t}}$. In particular, we the notation $\yt_{i,j,t}$ to refer to element $(i,j)$ of $\smash{\yt\in\R^{d\times\dt}}$, and use the notation $\smash{\yt_{\cdot,j,t}=(\yt_{i,j,t})_{i=1}^{d}}$. We assume here that $\smash{A,L\in\R^{\dt \times \dt}}$ are positive definite matrices and each of the $\smash{W_i=(W_{i,t})_{t\geq 0}}$ are independent $\dt$-dimensional Brownian Motions.

Next, we assume that we may write each element of a noisy gradient process as ${g_{i,t} = b^{\T} y_{i,\cdot,t} + \sigma \xi_{i,t}}$,
where $\sigma>0$ and where $\xi_{i}=(\xi_{i,t})_{t\geq 0}$ are independent white noise processes. Noting that $\smash{\E[\, \nabla_i f(X_{t+h}) \lvert \F_t ] = b^{\T} e^{-A h} y_{i,t}}$, we find that this model implicitly assumes that gradients are expected decrease in exponentially in magnitude as a function of time, at a rate determined by the eigenvalues of the matrix $A$. The parameters $\sigma$ and $L$ can be interpreted as controlling the scale of the noise within the observation and signal processes.

Using this model, we obtain that the filter can be expressed as $\smash{\E[\,\nabla_i f(X_t) {\lvert} \F_t] = b^{\T} \yt_{i,t}}$, where $\yt_{i,t} = \E[ y_{i,t} \lvert \F_t ]$. The process $\yt_{i,t}$ is expressed as the solution to the Kalman-Bucy\footnote{For information on continuous time filtering and the Kalman-Bucy filter we refer the reader to the text of \citet{bensoussan2004stochastic} or the lecture notes of \citet{van2007stochastic}.} filtering equations \useshortskip
\begin{equation} \label{eq:kalman-bucy-filter-dynamics}
	d\yt_{i,t} = -A \yt_{i,t} \, dt + \sigma^{-1} \Pt_{t} \, b \, d\Bt_{i,t}
	\;,\; \hspace{2em}
	\dot{\Pt} = -A \Pt_{t} - \Pt_{t}^{\T} A - \sigma^{-2} \Pt_{t} b \, b^{\T} \Pt_{t}^{\T}  + L L^\T
	\;, 
\end{equation}
with the initial conditions $\yt_{i,0}=0$ and $\smash{\Pt_0=\E[ y_{i,0} y_{i,0}^\T ]}$, and where we define innovations process $d\Bt_{i,t} = \sigma^{-1} \left( g_{i,t} - b^{\T}\yt_{i,t} \right)\, dt$ with the property that each $\Bt_i$ is an independent $\F$-adapted Brownian motion. 

Inserting the linear state space model and its filter into the optimality equations~\eqref{eq:Thm-Optimality-BSDE} we obtain the following result.

\begin{proposition}[State-Space Model Solution to the FOSP] 
\label{prop:perturbation-solution-filter}
	Assume that the gradient state-space model described above holds. The FOSP approximation to the solution of the optimality equations~\eqref{eq:Thm-Optimality-BSDE} can be expressed as \useshortskip
	\begin{equation} \label{eq:FOSP-Linear-Diffusion-Solution}
		dX_t = e^{\alpha_t} 
		{\mathLarger{(}} 
		\nabla h^{\ast} ( \nabla h(X_t) - 
		{\textstyle \sum_{j=1}^{\dt} } \tPhi_{j,t} \yt_{\cdot,j,t}
		) 
		- X_t^{\nust}
		{\mathLarger{)}} 
		\,dt
	\;,
	\end{equation}
	where
	$\tPhi_t =  
		e^{-\gamma_t} 
		(b^{\T} e^{-At} \Phi_0 + 
		\int_0^t 
		e^{\alpha_u + \beta_u + \gamma_u} b^{\T} e^{-A(t-u)}
		\, du) \in \R^{\dt}$ is a deterministic learning rate,
	 where $e^{A}$ represents the matrix exponential, and where $\Phi_0 = 
	 	e^{\delta_T} e^{AT}
	 	-
	 	\int_0^T 
	 	e^{\alpha_u + \beta_u+\gamma_u} e^{A u}
		\, du$
	 can be chosen to have arbitrarily large eigenvalues by scaling $\delta_T$. The martingale portion of the solution of~\eqref{eq:Thm-Optimality-BSDE} can be expressed as $\mcM_t = \mcM_0 - \sigma^{-1} \int_0^t e^{\alpha_u + \beta_u + \gamma_u } b^{\T} e^{-A(t-u)} \Pt_u b \, d\Bt_{u} $.
\end{proposition}
\vspace{-1.45em}
\begin{proof}
	See Appendix~\ref{sec:proof-prop-perturbation-solution-filter}
\end{proof}


\subsubsection{Kalman Gradient Descent}
\label{sec:kalman-gradient-descent}

In order to recover Kalman Gradient Descent, we discretize the processes $X_t^{\nust}$ and $\yt$ over the finite mesh $\mcT$, defined in equation~\eqref{eq:FOSP-Linear-Diffusion-Solution}. Applying a Forward-Euler-Maruyama discretization of~\eqref{eq:FOSP-Linear-Diffusion-Solution} and the filtering equations~\eqref{eq:kalman-bucy-filter-dynamics}, we obtain the discrete dynamics
\begin{equation} \label{eq:Kalman-Filtering-Equations}
	y_{i,t_{k+1}} = (I-e^{-\alpha_{t_k}} A) y_{i,t_k} + L e^{-\alpha_t} w_{i,k}
	\;,\hspace{3em}
	g_{i,t_k} = b^{\T} y_{i,t_k} + \sigma e^{-\alpha_t} \xi_{i,k}
	\;,
\end{equation}
where each of the $\xi_{i,k}$ and $w_{i,k}$ are standard Gaussian random variables of appropriate size. The filter $\smash{\yt_{i,k} = \E[ y_{t_k} {\lvert} \{g_{t_{k^{\prime}}}\}_{k^{\prime}=1}^k ]}$ for the discrete equations can be written as the solution to the discrete \emph{Kalman filtering equations}, provided in Appendix~\ref{sec:Kalman-Filtering-Equations}. Discretizing the process $X^{\nust}$ over $\mcT$ with the Forward-Euler scheme, we obtain discrete dynamics for the optimizer in terms of the \emph{Kalman Filter}~$\yt$, as \useshortskip
\begin{equation} \label{eq:Kalman-Descent}
	X_{t_{k+1}}
 	=
	\nabla h^{\ast} \left( \nabla h(X_{t_k}) - 
	{\textstyle \sum_{j=1}^{\dt} } \tPhi_{j,t_k} \yt_{\cdot,j,k} 
	\right) 
	\;,
\end{equation}
yielding a generalized version of Kalman gradient descent of \citet{vuckovic2018kalman} with $\dt$ states for each gradient element. Setting $h(x)=\frac{1}{2} \|x\|^2$, $\dt=1$ and $b=1$ recovers the original Kalman gradient descent algorithm with a time-varying learning rate. 

Just as in Section~\ref{sec:mini-batch-SGD-mirror-descent}, we interpret each $g_{t_k}$ as being a mini-batch gradient, as with equation~\eqref{eq:risk-gradient-sample-motivation}. The algorithm~\eqref{eq:Kalman-Descent} computes a Kalman filter from these noisy mini-batch observations and uses it to update the optimizer's position.

\subsubsection{Momentum and Generalized Momentum Methods}

By considering the asymptotic behavior of the Kalman gradient descent method described in Section~\ref{sec:kalman-gradient-descent}, we recover a generalized version of momentum gradient descent methods, which includes mirror descent behavior, as well as multiple momentum states. Let us assume that $\alpha_t=\alpha_0$ remains constant in time. Then, using the asymptotic update rule for the Kalman filter, as shown in Proposition~\ref{prop:Asymptotic-Kalman-Dynamics}, and equation~\eqref{eq:Kalman-Descent}, we obtain the update rule
\begin{equation} \label{eq:Asymptotic-Kalman-Descent-Rule}
	X_{t_{k+1}}
 	=
	\nabla h^{\ast} \left( \nabla h(X_{t_k}) - 
	{\textstyle \sum_{j=1}^{\dt} } \tPhi_{j,t_k} \yt_{\cdot,j,k} 
	\right) 
	\;, \hspace{2em}
	\yt_{i,\cdot,k}
	= \left( \tA - K_{\infty} b^{\T} \tA \right) \yt_{i,\cdot,k} + K_{\infty} g_{i,k}
	\;,
\end{equation}
where $\tA=I - e^{-\alpha_0} A$ and where $K_{\infty} \in \R^{\dt}$ is defined in the statement of the Proposition~\ref{prop:Asymptotic-Kalman-Dynamics}. This yields a generalized momentum update rule where we keep track of $\dt$ momentum states with $(\yt_{i,j,k})_{j=1}^{\dt}$, and update its position using a linear update rule. This algorithm can be seen as being most similar to the Aggregated Momentum technique of~\cite{lucas2018aggregated}, which also keeps track of multiple momentum states which decay at different rates.

Under the special case where $\dt=1$, $b=1$, and $h=\frac{1}{2}\|x\|^2$ we recover the exact momentum algorithm update rule of~\cite{polyak1964some} as 
\begin{equation}
	X_{t_{k+1}} - X_{t_k}
 	=
	 - \tPhi_{t_k} \yt_{k} 
	\;, \hspace{2em}
	\yt_{i,k}
	= p_1 \, \yt_{k} + p_2 \, g_{t_k}
	\;,
\end{equation}
where we have a scalar learning rate $\tPhi_{t_k}$, where $p_1=\tA - K_{\infty} b^{\T} \tA$, $p_2=K_{\infty}$ are positive scalars, and where $g_{t_k}$ are mini-batch draws from the gradient as in equation~\ref{eq:risk-gradient-sample-motivation}.

The recovery of the momentum algorithm of~\cite{polyak1964some} has some interesting consequences. Since $p_1$ and $p_2$ are functions of the model parameters $\sigma, A$ and $\alpha_0$, we obtain a direct relationship between the optimal choice for the momentum model parameters, the assumed scale of gradient noise $\sigma,L>0$ and the assumed expected rate of decay of gradients, as given by $e^{-At}$. This result gives insight as to how momentum parameters should be chosen in terms of their prior beliefs on the optimization problem.

\section{Discussion and Future Research Directions}

Over the course of the paper we present a variational framework on optimizers, which interprets the task of stochastic optimization as an inference problem on a latent surface that we wish to optimize. By solving a variational problem over continuous optimizers with asymmetric information, we find that optimal algorithms should satisfy a system of FBSDEs projected onto the filtration $\F$ generated by the noisy observations of the latent process.

By solving these FBSDEs and obtaining continuous-time optimizers, we find a direct relationship between the measure assigned to the latent surface and its relationship to how data is observed. In particular, assigning simple prior models to the pair of processes $(\nabla f(X_t) , g_t)_{t \in [0,T]}$, recovers a number of well-known and widely used optimization algorithms. The fact that this framework can naturally recover these algorithms begs further study. In particular, it is still an open question whether it is possible to recover other stochastic algorithms via this framework, particularly those with second-order scaling adjustments such as ADAM or AdaGrad.

From a more technical perspective, the intent is to further explore properties of the optimization model presented here and the form of the algorithms it suggests. In particular, the optimality FBSDE~\ref{eq:Thm-Optimality-BSDE} is nonlinear, high-dimensional and intractable in general, making it difficult to use existing FBSDE approximation techniques, so new tools may need to be developed to understand the full extent of its behavior.

Lastly, numerical work on the algorithms generated by this framework can provide some insights as to which prior gradient models work well when discretized. The extension of simplectic and quasi-simplectic stochastic integrators applied to the BSDEs and SDEs that appear in this paper also has the potential for interesting future work.

\bibliographystyle{plainnat}
\bibliography{references_stochoptim,biblio}

\clearpage


\appendix

\footnotesize

\section{Obtaining Solutions to the Optimality FBSDE}
\label{sec:Solution-Techniques-for-Optimality-FBSDE}

\subsection{A Momentum-Based Representation of the Optimizer Dynamics}

Using a simple change of variables we may represent the dynamics of the FBSDE~\eqref{eq:Thm-Optimality-BSDE} in a simpler fashion, which will aid us in obtaining solutions to this system of equations. Let us define the \textit{momentum process} $p=(p_t)_{t\in[0,T]}$ as
\begin{equation} \label{eq:momentum-definition}
	p_t 
	= \left(\frac{\partial \L}{\partial \nu}\right)_t
	= e^{\gamma_t} \left(\nabla h(X_t^{\nust} + e^{-\alpha_t} \nust ) - \nabla h(X_t^{\nust})\right)
	\;.
\end{equation}
Noting that since $h$ is convex, we have the property that $\nabla \hst (x) = (\nabla h)^{-1} (x)$, we may use equation~\eqref{eq:momentum-definition} to write $\nust$ in terms of the momentum process as
\begin{equation} \label{eq:nu-as-function-of-momentum}
	\nust = e^{-\alpha_t} \left( \nabla \hst \left( \nabla h(X_t) + e^{-\gamma_t} p_t \right) - X_t \right)
	\;.
\end{equation}

The introduction of this process allows us to represent the solution to the optimality FBSDE~\eqref{eq:Thm-Optimality-BSDE}, and by extension the optimizer, in a much more tractable way. Re-writing~\eqref{eq:Thm-Optimality-BSDE} in terms of $p_t$, we find that
\begin{equation} \label{eq:momentum-dynamics}
	\left\{
	\begin{aligned}
	 	d p_t
	 	&= - \left\{ 
	 	e^{\gamma_t + \alpha_t + \beta_t} 
	 	\E\left[ \nabla f(X_t^{\nust}) {\big \lvert} \F_t \right] 
	 	+ \left( 
	 	e^{\gamma_t} \nabla^2 h(X_t) \, \nust_t
	 	- e^{\alpha_t} p_t
	 	\right)
	 	\right\} \, dt + d\M_t
	 	\\
	 	p_T &= 
	 	- e^{\delta_T} \E\left[ \nabla f ( X_T^{\nust} ) {\big \lvert} \F_T \right]
	\end{aligned}
	\right.
\end{equation}
where the dynamics of the forward process $X^{\nust}$ can be expressed as
\begin{equation}
	dX_t^{\nust} = 
	e^{\alpha_t} \left( \nabla \hst \left( \nabla h(X_t^{\nust}) + e^{-\gamma_t} p_t \right) - X_t^{\nust} \right) \, dt
\;.
\end{equation}
This particular change of variables corresponds exactly to the Hamiltonian representation of the optimizer's dynamics, which we show in Appendix~\ref{sec:Hamiltonian-Dynamics}.

Writing out the explicit solution to the FBSDE~\eqref{eq:momentum-dynamics}, we obtain a representation for the optimizer's dynamics as
\begin{equation} \label{eq:momentum-integral-rep}
	p_t = 
	\E\left[ 
	\int_t^T 
	e^{\gamma_u}
	\left\{ 
 	e^{\alpha_u + \beta_u} 
 	\nabla f(X_u^{\nust}) 
 	+ \left( 
 	\nabla^2 h(X_u) \, \nust_u
 	- e^{\alpha_u - \gamma_u} p_u
 	\right)
 	\right\} \,du \, 
   - e^{\delta_T} \nabla f ( X_T^{\nust} ) 
	\; {\Big \lvert} \F_t \right]
	\;,
\end{equation}
showing that optimizer's momentum can be represented as a time-weighted average of the expected future gradients over the remainder of the optimization and the term $e^{\gamma_t} \nabla^2 h(X_t) \, \nust_t - e^{\alpha_t} p_t$, where the weights are determined by the choice of hyperparameters $\alpha,\beta$ and $\gamma$. Noting that
\begin{equation} \label{eq:taylor-correction-term}
 	\nabla^2 h(X_t) \, \nust_t - e^{\alpha_t - \gamma_t} p_t
 	= 
 	\nabla^2 h(X_t) \nust_t -  \left(\frac{\nabla h(X_t + e^{-\alpha_t} \nust_t) - \nabla h(X_t)}{e^{-\alpha_t}}\right)
 	\;,
\end{equation} 
we find that the additional correction term in~\eqref{eq:momentum-integral-rep} can be interpreted as the remainder in the first-order Taylor expansion of the term $\nabla h (X_t + e^{-\alpha_t} \nust )$.

The representation~\eqref{eq:momentum-integral-rep} demonstrates optimizer does not only depend on the instantaneous value of gradients at the point $X_t^{\nust}$. Rather, we find that the algorithm's behaviour depends on the expected value of all future gradients that will be encountered over the remainder of the optimization process, projected onto the set of accumulated gradient information, $\F_t$. This is in stark contrast to most known stochastic optimization algorithms which only make explicit use of local gradient information in order to bring the optimizer towards an optimum. 

\subsection{First-Order Singular Perturbation Approximation} 

When $h$ does not take the quadratic form $h(x)=\frac{1}{2} x^{\T} M x$ for some positive-definite matrix $M$, the nonlinear dynamics of the FBSDE~\eqref{eq:Thm-Optimality-BSDE} or in the equivalent momentum form~\eqref{eq:momentum-dynamics} make it difficult to provive a solution for general $h$. More precisely, the Taylor expansion term~\eqref{eq:taylor-correction-term} constitutes the main obstacle in obtaining solutions in general. 

In cases where the scaling parameter $\alpha_t$ is sufficiently large, we can assume that the Taylor expansion remainder term of equation~\eqref{eq:taylor-correction-term} will become negligibly small. Hence, we may approximate the optimality dynamics of the FBSDE~\eqref{eq:momentum-dynamics} by setting this term to zero. This can be interpreted as the first-order term in a singular perturbation expansion of the solution to the momentum FBSDE~\eqref{eq:momentum-dynamics}. 

Under the assumption that the Taylor remainder term vanishes, we obtain the approximation $\tpz=(\tpz)_{t\in[0,T]}$ for the momentum, which we present in the following proposition.

\begin{proposition}[First-Order Singular Perturbation (FOSP)] \label{prop:Perturbation-Solution}
	The linear FBSDE
	\begin{equation} \label{eq:Singular-Perturbation-FBSDE}
		\left\{
		\begin{aligned}
			&d \tpz_t
			=
			- e^{\gamma_t + \alpha_t + \beta_t} \, \E\left[ \nabla f\left(X_t\right)  \lvert \F_t \right]
			 \, dt + d\tMz_t
			\\ 
			&\tpz_T = - e^{\delta_T} \E\left[ \nabla f ( X_T^{\nust} ) {\big \lvert} \F_T \right]
		\end{aligned}
		\right.
		\;,
	\end{equation}
	admits a solution that can be expressed as
	\begin{equation} \label{eq:FOSP-Solution}
		\tpz_t
		=
		\E\left[\left.
		\int_{t}^{T}
		e^{\gamma_u + \alpha_u + \beta_u}   
		\, \nabla f\left( X_u \right) \, du
		- e^{\delta_T} \nabla f(X_T^{\nust} )
		\right\lvert  \F_t \right]
		\;,
	\end{equation}
	provided that $\E\left[
		\int_{0}^{T}
		e^{\gamma_u + \alpha_u + \beta_u}   
		\, \lVert \nabla f\left( X_u \right) \lVert \, du
		\right] < \infty$. 
\end{proposition}
\begin{proof}
	Noting that the remainder term in the expression~\eqref{eq:taylor-correction-term} vanishes, we get that
\begin{equation} \label{eq:proof-solution-singular-perturbation}
	\tpz_t
	= 
	\E\left[
	\int_t^T
	e^{\gamma_u + \alpha_u +\beta_u} \, \nabla f\left(X_u\right)
	\, du
	-
	e^{\delta_T} \, \nabla f(X_T^{\nust})
	{\Big \lvert} \F_u \right]
	\;.
\end{equation}

Under the assumption that $\alpha,\beta,\delta,\gamma$ are continuous over $[0,T]$ and that $\E \| f(x) \|^2\| < \infty$, the right part of~\eqref{eq:proof-solution-singular-perturbation} is bounded. Now note that the integral on the left side of~\eqref{eq:proof-solution-singular-perturbation} is upper bounded for all $T$ by the integral provided in the integrability condition of Proposition~\ref{prop:Perturbation-Solution}, and therefore this condition is a sufficient condition for the expression~\eqref{eq:proof-solution-singular-perturbation} to be finite and well-defined.

\end{proof}

Although a general, model independent bound for the accuracy of such approximations is beyond the scope of this paper, it can still serve as a reasonable and computationally cheap alternative to attempting to solve the original problem dynamics directly with a BSDE numerical scheme. For more information on singular perturbation methods in the context of FBSDEs, see~\cite{jankovic2012perturbed}.

\subsection{Hamiltonian Representation of the Optimizer Dynamics}
\label{sec:Hamiltonian-Dynamics}

Just as in Hamiltonian classical mechanics, it is possible to express the optimality FBSDE of Theorem~\eqref{thm:Stochastic-Euler-Lagrange} with Hamiltonian equations of motion. We define the Hamiltonian $\H$ as the Legendre dual of $\L$ at, which can be written as
\begin{equation} \label{eq:hamiltonian-expression}
	\H(t,X,p)
	=
	\left\langle p \,  , \nust \right\rangle
	- \L(t,X,\nust)
\;,
\end{equation}
where $p = \frac{\partial \L}{\partial X}$.
Using the identity $D_h(x,y) = D_{\hst}(\nabla h(x), \nabla h(y))$, where $\hst$ is the Legendre dual of $h$, and inverting the expression for $\frac{\partial \L}{\partial X}$ in terms $p$, we may compute equation~\eqref{eq:hamiltonian-expression} as\footnote{See~\cite{wibisono2016variational}[Appendix B.4] for the full details of the computation.}
\begin{equation}
	\H(t,X,p) \label{eq:explicit-Hamiltonian-expression}
	=
	e^{\alpha_t+\gamma_t} D_{\hst} \left( \nabla h(X) + e^{-\gamma_t} p 
	\,,\, 
	\nabla h(X)
	\right)
	+ e^{\gamma_t + \beta_t} f(X_t)
	\;.
\end{equation}
Using this definition of $\H$, and using the FBSDE~\eqref{eq:Thm-Optimality-BSDE}, we obtain the following equivalent representation for the dynamics of the optimizer.

Using the simple substitution $p_t=\left(\frac{\partial \L}{\partial X}\right)_t$ and noting from equations~\eqref{eq:BSDE-dLdX} and~\eqref{eq:BSDE-dLdnu} that
\begin{equation}
	p_t = e^{\gamma_t} \left(\nabla h(X_t + e^{-\alpha_t} \nust_t) - \nabla h(X_t)\right)
	\;,
\end{equation}
a straightforward computation applied to the definition of $\H$ shows that the dynamics of the optimality FBSDE~\eqref{eq:Thm-Optimality-BSDE} admit the alternate Hamiltonian representation of the optimizer dynamics
\begin{equation} \label{eq:optimality-BSDE-h}
	dX_t = \left(\frac{\partial \H}{\partial p}\right)_t  dt
	\hspace{0.5em},\hspace{1.5em}
	d p_t
	=
	- \E\left[ \left(\frac{\partial \H}{\partial X}\right)_t {\Big \lvert} \F_t \right] \, dt - d\M_t
\end{equation}
along with the boundary condition $p_{T}=0$. 

\section{The Discrete Kalman Filter}
\label{sec:Kalman-Filtering-Equations}

Here we present the reader to the Kalman Filtering equations used in Section~\ref{sec:Kalman-Algorithms-Section}. Consider the model presented in equations~\eqref{eq:Kalman-Filtering-Equations},
\begin{equation} \label{eq:Kalman-Filtering-Equations-2}
	y_{i,t_{k+1}} = \tA_k y_{i,t_k} + \tL_k w_{i,k}
	\;,\hspace{3em}
	g_{i,t_k} = b^{\T} y_{i,t_k} + \sigma e^{-\alpha_t} \xi_{i,k}
	\;,
\end{equation}
where we use the notation $\tA_k=(I-e^{-\alpha_{t_k}} A)$ and $\tL_k = L e^{-\alpha_t}$, and where $w_{i,k}$ and $\xi_{i,k}$ are all independent standard Gaussian random variables. We provide the Kalman filtering equations for this model in the following proposition.

\begin{proposition}[{\citet[Theorem 10.2]{walrand-dimakis-randomProcesses}}]
	Let $\yt_{i,k} = \E[ y_{t_k} {\lvert} \sigma(g_{t_{k^{\prime}}})_{k^{\prime}=1}^k ]$. Then $\yt_{i,k}$ satisfies the recursive equation
\begin{equation} \label{eq:Kalman-Filter-Expression-Discrete}
	\yt_{i,k} = \tA_k \yt_{i,k} + K_k \left( g_{i,k} - b^{\T} \tA_k \yt_{i,k} \right)
	\;,
\end{equation}
where the matrices $K_k$ are obtained via the independent recursive equations
\begin{align}
	{P}_{k\mid k-1} &= \tA_k {P}_{k-1\mid k-1} \tA_k^\T + \tL_k^{\T} \tL_k
	\,, \\
	{S}_k &= \sigma^{2} + b^{\T} {P}_{k\mid k-1}  b 
	\,, \\
	{K}_k &= {P}_{k\mid k-1} b \, {S}_k^{-1}
	\,, \\
	{P}_{k|k} &= \left({I} - {K}_k b^{\T}\right) {P}_{k|k-1} 
	\,.
\end{align}
\end{proposition}
For more information on the discrete Kalman filter, its derivation and for asymptotic properties, we refer the reader to the lecture notes~\cite{walrand-dimakis-randomProcesses}.

Next, we provide a result on the asumptotic properties of the Kalman filter in the proposition that follows.
\begin{proposition}[{\citet[Theorem 11.2]{walrand-dimakis-randomProcesses}}]
\label{prop:Asymptotic-Kalman-Dynamics}
	Assume that $\alpha_{t_k}=\alpha_{t_0}$ is constant, so that $\tA_k=\tA$ and $\tL_k = \tL$ become constant, and assume that there exists a positive-definite solution $K_{\infty}\in\R^{\dt \times \dt}$ to the algebraic matrix equation
\begin{equation}
	\tilde{K} = \tA \tilde{K} \tA^{\T} + \tL \tL^{\T}
	\;.
\end{equation}
Then, we may write the asymptotic dynamics of the filter $\yt_{i}$ as
\begin{equation}
	\yt_{i,k} = \tA \yt_{i,k} + K_{\infty} \left( g_{i,k} - b^{\T} \tA \yt_{i,k} \right)
	\;,
\end{equation}
where $K_{\infty}$ is the solution to the system of algebraic matrix equations
\begin{equation}
	K_{\infty} = (I-RC) S \,,\;\;
	R=S b \left(b^{\T} S b + \sigma^{2} \right)^{-1} \,,\;\;
	S=\tA K_{\infty} \tA^{\T} + \tL \tL^{\T}
	\,.
\end{equation}
\end{proposition}
For more information on the Kalman Filter, its derivation and theoretical properties, see~\cite{walrand-dimakis-randomProcesses}.

\clearpage

\section{Proofs Relating to Theorem~\ref{thm:Stochastic-Euler-Lagrange}}
\label{sec:proof-thm-Stochastic-Euler-Lagrange}

Before going forward with the main part of the proof, we first present a lemma for the computation of the G\^ateaux derivative of $\J$.

\begin{lemma} \label{lemma:alternative-J-expr}
	The functional $\J$ is everywhere G\^ateaux differentiable in $\A$. The G\^ateaux at a point $\nu\in\A$ in the direction $\tomega = \omega - \nu$ for $\omega\in\A$ takes the form
	\begin{equation} \label{eq:proof-gateaux-deriv-expr-2}
		\left\langle D\J(\nu) , \tomega \right\rangle
		=
		\E\left[
		\int_0^{T}
		\left\langle \omega_t \,,\;
		\frac{\partial \L\left( t,X_t^{\nu}, \nu_t \right)}{\partial \nu}
		-
		\E\left[
		\int_t^{T}
		\frac{\partial \L\left( u,X_u^{\nu}, \nu_u \right)}{\partial X} 
		\, du
		- e^{\delta_T} \, \nabla f\left( X_T^{\nu} \right)
		\Big\lvert \F_t
		\right]
		\right\rangle
		\,dt
		\right]
		\;.
		\end{equation}
\end{lemma}
\begin{proof}
	If we assume that the conditions of Leibniz' rule hold, we may compute the G\^ateax derivative as 
	\begin{align}
		\partial_{\rho} \J\left(\nu+\rho\,\tomega\right)
		&=
		\partial_{\rho} \E\left[
		\int_0^T \L\left(t,X_t^{\nu + \rho\,\tomega}, \nu_t + \rho\,\tomega_t \right) \, dt
		+ e^{\delta_T} \left(f(X_T^{\nu + \rho\,\tomega}) - f(\xst) \right)
		\right] \nonumber
		\\ &=
		\E\left[
		\int_0^T 
		\partial_{\rho} 
		\L\left(t,X_t^{\nu + \rho\,\tomega}, \nu_t + \rho\,\tomega_t \right) \, dt
		+ e^{\delta_T} \partial_\rho f(X_T^{\nu + \rho\,\tomega})
		\right] \nonumber
		\\ &=
		\E\left[
		\int_0^T 
		\left\{
		\left\langle 
		\frac{\partial \L\left( t,X_t^{\nu}, \nu_t \right)}{\partial X} 
		\,,\;
		\int_0^t \tomega_u \, du
		\right\rangle
		+
		\left\langle 
		\frac{\partial \L\left( t,X_t^{\nu}, \nu_t \right)}{\partial \nu}
		\,,\;
		\tomega_t
		\right\rangle
		\right\}
		\,dt
		+ \left\langle \int_0^T \tomega_u \, du  , \Phi \nabla f(X_T^{\nu}) \right\rangle
		\right] 
		\label{eq:proof-leibniz-gateaux}
		\;,
	\end{align}
	where we have
	\begin{align}	
	\label{eq:BSDE-dLdX-thm-proof}
	\frac{\partial \L\left( t,X, \nu \right)}{\partial X} 
	&=
	e^{\gamma_t+\alpha_t} \mathLarger{(} \,
	\nabla h(X + e^{-\alpha_t} \nu )  - \nabla h(X)
	 -  e^{-\alpha_t} \nabla^2 h(X) \nu  - e^{\beta_t} \nabla f(X)
	\, \mathLarger{)}
	\\
	\frac{\partial \L\left( t,X, \nu \right)}{\partial \nu}
	&=
	e^{\gamma_t} \left(
	\nabla h(X + e^{-\alpha_t} \nu ) - \nabla h(X)
	\right)
	\;.
	\end{align}
	Note here that the derivative in $f$ is path-wise for every fixed realization of the function $f$. Since $f\in C^1$, we have that $\nabla f$ is also well-defined for every realization of $f$.

	To ensure that this computation is valid, and that the conditions of the Leibniz rule are met, due to the continuity of~\eqref{eq:proof-leibniz-gateaux} in $\tomega$, is sufficient for us to show that the integrals in equation~\eqref{eq:proof-leibniz-gateaux} are bounded for any $\tomega$ and $\nu$. First, note that by the Young and Jensen inequalities,
	\begin{align}
		\E\left[\left\langle \int_0^T \tomega_u \, du  , \Phi \nabla f(X_T^{\nu}) \right\rangle\right]
		&\leq
		\frac{1}{2}
		\E\left[ \int_0^T \|\tomega_u\|^2 \, du  +  \Phi \|\nabla f(X_T^{\nu}) \|^2
		\right] < \infty
		\;,
	\end{align}
	where the boundedness holds from the fact that $\tomega\in\A$ and that $\E \|f(x)\|^2 < \infty$ for all $x\in\R^d$.

	Next, we focus on the left part of equation~\eqref{eq:proof-leibniz-gateaux}.  By the Cauchy-Schwarz and Young inequalities, we have
	\begin{align}
		\left\lvert
		\left\langle 
		\frac{\partial \L\left( t,X_t^{\nu}, \nu_t \right)}{\partial X} 
		\,,\;
		\int_0^t \tomega_u \, du
		\right\rangle
		+
		\left\langle 
		\frac{\partial \L\left( t,X_t^{\nu}, \nu_t \right)}{\partial \nu}
		\,,\;
		\tomega_t
		\right\rangle
		\right\rvert
		&\leq
		\left\| \frac{\partial \L\left( t,X_t^{\nu}, \nu_t \right)}{\partial X} \right\| 
		\left\| \int_0^t \tomega_u \, du  \right\|
		+
		\left\|\frac{\partial \L\left( t,X_t^{\nu}, \nu_t \right)}{\partial \nu} \right\|
		\|\tomega_t\|
		\\ &\leq
		\frac{1}{2}\left\{
		\left\| \int_0^t \tomega_u \, du  \right\|^2
		+
		\left\| \frac{\partial \L\left( t,X_t^{\nu}, \nu_t \right)}{\partial X} \right\|^2
		+
		\| \frac{\partial \L\left( t,X_t^{\nu}, \nu_t \right)}{\partial \nu} \|^2
		+
		\|\tomega_t\|^2
		\right\}
		\;.  \label{eq:proof-integrand-Gateaux}
 	\end{align}
	Using the $L$-Lipschitz property of the gradients of $h$, we can also bound the partial derivatives of the Lagrangian with the triangle inequality as
	\begin{align*}
		\left\| \frac{\partial \L\left( t,X_t^{\nu}, \nu_t \right)}{\partial X} \right\| 
		&\leq 
		e^{\gamma_t+\alpha_t} \left\| \nabla h(X + e^{-\alpha_t} \nu )  - \nabla h(X) \right\| + 
		e^{\gamma_t} \| \nabla^2 h(X) \nu \| + 
		e^{\beta_t+\gamma_t+\alpha_t} \| \nabla f(X)\|
		\\ &\leq
		L (e^{\gamma_t+\alpha_t} + e^{\gamma_t}) \| \nu \| 
		+ e^{\beta_t+\gamma_t+\alpha_t} \| f(X)\|
		\\ &\leq 
		C _0\left( \| \nu \| + \| \nabla f(X)\| \right)
		\\ 
		\left\| \frac{\partial \L\left( t,X, \nu \right)}{\partial \nu} \right\|
		&\leq
		e^{\gamma_t} \left\|
		\nabla h(X + e^{-\alpha_t} \nu ) - \nabla h(X)
		\right\|
		\\ &\leq
		e^{\gamma_t} L \left\| \nu \right\|
		\\ &\leq
		C_0 \, \left\| \nu \right\|
		\;,
	\end{align*}
	where $C_0 = \sup_{t\in [0,T]} \{ e^{\alpha_t + \gamma_t} + e^{\gamma_t} + e^{\alpha_t+\gamma_t+\beta_t}\}$ is bounded by the assumption that $\alpha,\beta,\gamma$ are continuous in $[0,T]$.
	
	Using the above result, and applying Young's inequality to the previous result, we can upper bound equation~\eqref{eq:proof-integrand-Gateaux} as
	\begin{align}
		\text{\eqref{eq:proof-integrand-Gateaux}}
		&\leq 
		32 \, (1 + C) \, \left\{ 1 +
		\int_0^T \| \tomega_u \|^2 \, du
		+
		\| \nu_t \|^2
		+
		\|\tomega_t\|^2
		+
		\| \nabla f\left(X_t\right)\|^2
		\right\}
		\\&\leq 
		64 \, (1 + C) \, \left\{ 1 +
		\int_0^T \| \omega_u \|^2 \, du
		+
		\int_0^T \| \nu_u \|^2 \, du
		+
		\| \nu_t \|^2
		+
		\|\omega_t\|^2
		+
		\| \nabla f\left(X_t\right)\|^2
		\right\}
		\label{eq:upper-bound-integrand}
		 \;,
	\end{align}
	where the number 32 is chosen to be much larger than what is strictly necessary by Young's inequality.
	Notice here that by the definition of $\A$, this forms an integrable upper bound to the left integral of equation~\eqref{eq:proof-leibniz-gateaux}, validating our use of Leibniz's rule, and showing that $\J$ is indeed G\^ateaux integrable.

	Now that integrability concerns have been dealt with, we can proceed with the computation of the G\^ateaux derivative. By applying integration by parts to the left side of equation~\eqref{eq:proof-gateaux-deriv-expr} and moving the right hand side into the integral, we obtain
	\begin{equation*}
		\partial_{\rho} \J\left(\nu+\rho\,\tomega\right)
		=
		\E\left[
		\int_0^T 
		\left\langle \tomega_t \,,\;
		\frac{\partial \L\left( t,X_t^{\nu}, \nu_t \right)}{\partial \nu}
		-
		\int_t^T
		\frac{\partial \L\left( u,X_u^{\nu}, \nu_u \right)}{\partial X} 
		\, du
		- e^{\delta_T} \nabla f(X_T^{\nu})
		\right\rangle
		\,dt
		\right]
	\end{equation*}
	Using the tower property and Fubini's theorem on the right, we get
	\begin{equation} \label{eq:proof-gateaux-deriv-expr}
	\left\langle D\J(\nu) , \tomega \right\rangle
	=
	\E\left[
	\int_0^T 
	\left\langle \tomega_t \,,\;
	\frac{\partial \L\left( t,X_t^{\nu}, \nu_t \right)}{\partial \nu}
	-
	\E\left[
	\int_t^T
	\frac{\partial \L\left( u,X_u^{\nu}, \nu_u \right)}{\partial X} 
	\, du
	+ e^{\delta_T} \nabla f(X_T^{\nu}) \,
	\Big\lvert \F_t
	\right]
	\right\rangle
	\,dt
	\right]
	\;,
	\end{equation}
	as desired.
\end{proof}
\vspace{3em}

\subsection{Proof of Theorem~\ref{thm:Stochastic-Euler-Lagrange}}
Using the representation of the G\^ateux derivative of $\J$ brought forth by Lemma~\ref{lemma:alternative-J-expr}, we may proceed with the proof of Theorem~\ref{thm:Stochastic-Euler-Lagrange}.

\begin{proof}[Proof of Theorem~\ref{thm:Stochastic-Euler-Lagrange}]
	
	The goal is to show that the BSDE~\eqref{eq:Thm-Optimality-BSDE} is a necessary and sufficient condition for $\nust$ to be a critical point of $\J$. For any G\^ateaux differentiable function $\J$, a necessary and sufficient condition for a point $\nust\in\A$ to be a critical point is that its G\^ateaux derivative vanished in any valid direction. Lemma~\ref{lemma:alternative-J-expr} shows that the G\^ateaux derivative takes the form of equation~\eqref{eq:proof-gateaux-deriv-expr-2}. Therefore, all that remains is to show that the FBSDE~\ref{eq:Thm-Optimality-BSDE} is a necessary and sufficient condition for equation~\eqref{eq:proof-gateaux-deriv-expr-2} to vanish.

	\textbf{Sufficiency. }
	We will show that equation~\eqref{eq:proof-gateaux-deriv-expr-2} vanishes when the FBSDE~\eqref{eq:Thm-Optimality-BSDE} holds. Assume that there exists a solution to the FBSDE~\eqref{eq:Thm-Optimality-BSDE} satisfying $\nust\in\A$. We may then express the solution to the FBSDE explicitly as
	\begin{equation*}
		\left(\frac{\partial \L}{\partial \nu}\right)_t = \E\left[ \int_t^T \left(\frac{\partial \L}{\partial X}\right)_u \, du \, - e^{\delta_T} \nabla f(X_T^{\nu}) \, \Big\lvert \F_t \right] 
		\;.
	\end{equation*}
	Inserting this into the right side of~\eqref{eq:proof-gateaux-deriv-expr-2}, we find that $\left\langle D\J(\nu) , \omega \right\rangle$ vanishes for all $\omega\in\A$, demonstrating sufficiency.

	\textbf{Necessity. }
	Conversely, let us assume that $\left\langle D\J(\nu) , \omega - \nu \right\rangle = 0$ for all $\omega\in\A$ and for some $\nu\in\A$ for which the FBSDE~\eqref{eq:Thm-Optimality-BSDE} is not satisfied. We will show by contradiction that this statement cannot hold by choosing a direction in which the G\^ateax derivative does not vanish. Consider the choice
	\begin{equation} \label{eq:proof-gateaux-deriv-counterexample}
		\omega_t^{\rho} =
		\nu_t + \rho \,\left(
		\frac{\partial \L\left( t,X_t^{\nu}, \nu_t \right)}{\partial \nu}
		-
		\E\left[
		\int_t^T
		\frac{\partial \L\left( u,X_u^{\nu}, \nu_u \right)}{\partial X} 
		\, du - e^{\delta_T} \nabla f(X_T^{\nu})
		\Big\lvert \F_t \right]
		\right)
		\;,
	\end{equation}
	for some sufficiently small $\rho>0$. We will first show that $\omega^\rho \in \A$ for some $\rho>0$.

	First, note that clearly $\omega^{\rho}$ must be $\F_t$-adapted, and we have $\omega^{0} = \nu_t$. Moreover, note that since $\nu\in\A$, we have that $\E\int_0^T \, \lVert \nu_t \lVert^2 + \lVert \nabla f( X^{\nu} ) \lVert^2 \, dt < \infty$, that $\omega^{0}=\nu$. Notice that by the continuity of $\nabla f$ and the definition of $X$, the expression
	\begin{equation} \label{eq:admissible-condition-rho}
		\E \int_0^T \, \lVert \omega^{\rho}_t \lVert^2 + \lVert \nabla f( X^{\omega^{\rho}} ) \lVert^2 \, dt
	\end{equation}
	is continuous in $\rho$. Since \eqref{eq:admissible-condition-rho} is bounded for $\rho=0$, by continuity there exists some $\rho > 0$ for which \eqref{eq:admissible-condition-rho} is bounded and by extension where $\omega^{\rho}\in\A$ for this same value of $\rho$.

	Inserting~\eqref{eq:proof-gateaux-deriv-counterexample} into the G\^ateaux derivative~\eqref{eq:proof-gateaux-deriv-expr-2}, we get that
	\begin{equation}
	\left\langle D\J(\nu) , \omega^{\rho} - \nu \right\rangle
	=
	\rho \;
	\E \left[ \int_0^T 
	\left\lVert
	\frac{\partial \L\left( t,X_t^{\nu}, \nu_t \right)}{\partial \nu}
	-
	\E\left[
	\int_t^T
	\frac{\partial \L\left( u,X_u^{\nu}, \nu_u \right)}{\partial X} 
	\, du
	- e^{\delta_T} \nabla f(X_T^{\nu}) \, 
	\Big\lvert \F_t \right]
	\right\rVert^2
	\, dt
	\right]
	\;,
	\end{equation}
	which is strictly positive unless the FBSDE~\eqref{eq:Thm-Optimality-BSDE} is satisfied, thus forming a contradiction and demonstrating that the condition is necessary.
\end{proof}	

\section{Proof of Theorem~\ref{thm:convergence-rate}}
\label{sec:proof-thm-convergence-rate}
\begin{proof}
{ \footnotesize

The proof of this theorem is broken up into multiple parts. The idea will be to first show that the energy functional $\mcE$ is a super-martingale with respect to $\F_t$, and then to use this property to bound the expected distance to the optimum. Lastly, we bound a quadratic co-variation term which appears within these equations to obtain the final result.

Before delving into the proof, we introduce standard notation for semi-martingale calculus. We use the noation $dY_t = dY_t^c + \Delta Y_t$ to indicate the increments of the continuous part $Y^c$ of a process $Y$ and its discontinuities $\Delta Y_t = Y_t - Y_{t-}$, where we use the notation $t-$ to indicate the left limit of the process. We use the notation $[Y,Z]_t$ to represent the quadratic co-variation of two processes $Y$ and $Z$. This quadratic variation term can be decomposed into $d[Y,Z]_t = d[Y,Z]_t^c + \langle \Delta Y_t , \Delta Z_t\rangle$, where $[Y,Z]_t^c$ represents the quadratic covariation between $Y^c$ and $Z^c$, and where $\langle \Delta Y_t , \Delta Z_t\rangle$ represents the inner product of their discontinuities at $t$. For more information on semi-martingale calculus and the associated notation, see Jacod and Shiryaev~\cite[Sections 3-5]{jacod2013limit}.

\textbf{Dynamics of the Bregman Divergence.} The idea will now be to show that the energy functional $\mcE$, defined in equation~\eqref{eq:Energy-Functional}, is a super-martingale with respect to the visible filtration $\F_t$. 

Using It\^o's formula and It\^o's product rule for c\`adl\`ag semi-martingales~\cite{jacod2013limit}[Theorem 4.57], as well as the short-hand notation $Y_t = X_t + e^{-\alpha_t} \nust_t$, we obtain
{ \scriptsize
\begin{equation*}
	\begin{aligned}
		dD_h(\xst,Y_t) &=
		-\left\{
		\langle \nabla h(Y_t), dY_t^c \rangle
		+ \frac{1}{2} \sum_{i,j=1}^d \frac{\partial^2 h(Y_t)}{\partial x_i \partial x_i} d\left[ Y_i, Y_j \right]_t^c
		+ \Delta h(Y_t)
		\right\}
		- \left\{ 
		\vphantom{\sum_{i,j=1}^d}
		\langle d \nabla h(Y_t) , \xst - Y_t\rangle
		-
		\langle \nabla h(Y_t) , dY_t\rangle
		- d\left[ \nabla h(Y) , Y \right]_t
		\right\}
		\\ &=
		-\left\{
		\langle \nabla h(Y_t), -\Delta Y_t \rangle
		+ \frac{1}{2} \sum_{i,j=1}^d \frac{\partial^2 h(Y_t)}{\partial x_i \partial x_i} d\left[ Y_i, Y_j \right]_t^c
		+ \Delta h(Y_t)
		\right\}
		- \left\{
		\langle d \nabla h(Y_t) , \xst - Y_t\rangle
		- \sum_{i,j=1}^d \frac{\partial^2 h(Y_t)}{\partial x_i \partial x_i} d\left[ Y_i, Y_j \right]_t^c 
		- \langle \Delta \left( \nabla h(Y_t) \right) , \Delta Y_t \rangle
		\right\}
		\\ &
		= -
		\left\{
		\Delta h(Y_t) -
		 \langle \nabla h(Y_t) , \Delta Y_t \rangle
		\right\}
		- \langle d \nabla h(Y_t) , \xst - Y_t\rangle
		+ \left\{ \frac{1}{2} \sum_{i,j=1}^d \frac{\partial^2 h(Y_t)}{\partial x_i \partial x_i} d\left[ Y_i, Y_j \right]_t^c + 
		\langle \Delta \left( \nabla h(Y_t) \right) , \Delta Y_t \rangle \right\}
		\;,
	\end{aligned}
\end{equation*}
}where from line 1 to 2, we use the identity $d[\nabla g(Y),Y]_t=\sum_{i,j} \frac{\partial^2 g(Y_t )}{\partial x_i \partial x_j} d[Y_i,Y_j]_t^c + \langle \Delta( \nabla g(Y_t)), \Delta Y_t \rangle$ for any $C^2$ function $g$.

Note that since $h$ is convex, $\nabla^2 h$ must have positive eigenvalues, and hence $\frac{1}{2} \sum_{i,j=1}^d \frac{\partial^2 h(Y_t)}{\partial x_i \partial x_i} d\left[ Y_i, Y_j \right]_t^c \geq 0$. The convexity of $h$ also implies that $\langle \nabla h(x) - \nabla h(y), x-y\rangle \leq 0$, and therefore we get $\langle \Delta \left( \nabla h(Y_t) \right) , \Delta Y_t \rangle \geq 0$. The convexity of $h$ also implies that $\Delta h(Y_t) -\langle \nabla h(Y_t) , \Delta Y_t \rangle \geq 0$. Combining these observations, we find that
\begin{align}
		dD_h(\xst,Y_t) &\leq
		- \langle d \nabla h(Y_t) , \xst - Y_t\rangle
		+ \left\{ \sum_{i,j=1}^d \frac{\partial^2 h(Y_t)}{\partial x_i \partial x_i} d\left[ Y_i, Y_j \right]_t^c + 
		\langle \Delta \left( \nabla h(Y_t) \right) , \Delta Y_t \rangle \right\}
		\\ &=
		- \langle d \nabla h(Y_t) , \xst - Y_t\rangle
		+ [ \nabla h(Y), Y ]_t
		\;.
		\label{eq:dynamics-upper-bound}
\end{align}

\textbf{Super-martingale property of $\mcE$. }
Applying the scaling conditions to the optimality FBSDE~\eqref{eq:Thm-Optimality-BSDE}, we obtain the dynamics
\begin{equation}
	d  \nabla h(X_t^{\nust} + e^{-\alpha_t} \nust )
	 	= - e^{\alpha_t + \beta_t} \E\left[ \nabla f(X_t^{\nust}) {\big \lvert} \F_t \right] \, dt + d\tM_t
	 	\;.
\end{equation}
Inserting this in to the dynamics of for the energy functional, and applying the upper bound~\eqref{eq:dynamics-upper-bound}, we find that
\begin{align} \label{eq:proof-energy-dynamics-simplified-1}
	 d\mcE_t \leq&
	 - \langle d \nabla h(Y_t) , \xst - Y_t \rangle
	 + \dot \beta_t e^{\beta_t} \left( f(X_t) - f(\xst) \right) \, dt + 
	e^{\beta_t} \left\langle \nabla f(X_t), \nu_t \right\rangle \, dt
	\\
	=
	&\left\langle e^{\alpha_t + \beta_t } \E[\nabla f\left(X_t\right) \lvert \F_t ] \, dt - d\M_t
	\;,\;
	\xst - Y_t
	\right\rangle
	+ \dot \beta_t e^{\beta_t} \left( f(X_t) - f(\xst) \right) \, dt + 
	e^{\beta_t} \left\langle \nabla f(X_t), \nu_t \right\rangle \, dt
	\\
	=&
	- \left\{
	D_f(\xst, Y_t)
	+ \left(  e^{\alpha_t} - \dot \beta_t \right) e^{\beta_t} \left( f(X_t) - f(\xst)\right)\right\} \, dt
	+ d\M_t^{\prime}
	\;, \label{eq:proof-energy-dynamics-simplified-2}
\end{align}
where we use the notation $\M^\prime_t$ to represent the $\F_t$-martingale defined as
\begin{equation} 
	d\M_t^{\prime} = 
	\left\langle  
	e^{\alpha_t + \beta_t } \left(\E[\nabla f\left(X_t\right) \lvert \F_t ] - f(X_t) \right)\, dt - d\M_t
	\;,\;
	\xst - Y_t
	\right\rangle
	\;.
\end{equation}

Now note that due to the assumed convexity of $f$, we have that $D_f(\xst, Y_t)$ is almost surely non-negative. Second, by the scaling conditions, $e^{\alpha_t} - \dot \beta_t$ is positive. Hence, the drift in equation~\eqref{eq:proof-energy-dynamics-simplified-2} is almost surely negative, and $\mcE_t$ is a super-martingale.

Using the super-martingale property, we find that $\E\left[ \mcE_t \right] \leq \E\left[ \mcE_0 \right] 
		 = \E \left[ D_h(\xst,X_0 + e^{-\alpha_0} \nu_0) + e^{\beta_0} \left( f(X_0) - f(\xst)\right)\right]
		= C_0
		\;,$ where $C_0 \geq 0$. Using the definition of $\mcE$, and using the fact that $D_h\geq 0$ if $h$ is convex, we obtain
\begin{equation} \label{eq:proof-energy-dynamics-simplified-3}
 	e^{\beta_t} \E\left[ \left( f(X_t) - f(\xst)\right)\right]
 	\leq
	\E \left[ D_h(\xst,X_t + e^{-\alpha_t} \nu_t) + e^{\beta_t} \left( f(X_t) - f(\xst)\right) \right] \leq 
	C_0 + \E\left[ \, [ \nabla h(Y), Y  ]_t \, \right]
	\;.
\end{equation}

\textbf{Upper bound on the Quadratic Co-variation. }
Now we upper bound the quadratic co-variation term appearing on the right hand side of~\eqref{eq:proof-energy-dynamics-simplified-3}. Using the further change of variable $Z_t = \nabla h(Y_t)$, and noting that by the assumed convexity of $h$ that $\nabla \hst (x) = (\nabla h)^{-1}(x)$, we get $[ \nabla h(Y), Y  ]_t = [Z,\nabla \hst (Z)]_t$. 

Assuming that $\nabla h$ is $\mu$-strongly convex, we get that $\nabla \hst$ must have $\mu^{-1}$-Lipschitz smooth gradients. This implies that (i) the eigenvalues of $\nabla^2 \hst$ must be bounded above by $\mu^{-1}$ (ii) from the Cauchy-Schwarz inequality, we have $\langle \nabla \hst (x) - \nabla \hst(y), x - y\rangle \leq \mu^{-1} \|x-y\|^2$. Using these two observations and writing out the expression for $[Z,\nabla \hst (Z)]_t$, we get
\begin{align}
	[Z,\nabla \hst (Z)]_t
	&=
	\sum_{i,j=1}^d \frac{\partial^2 h(Y_t)}{\partial x_i \partial x_i} d\left[ Y_i, Y_j \right]_t^c
	+ \langle \Delta (\nabla \hst (Z)) , \Delta Z_t \rangle
	\\ &\leq 
	\mu^{-1} [Z]_t
	\;.
\end{align}
Moreover, note that since $Z_t = \nabla h(X_t^{\nust} + e^{-\alpha_t} \nust_t)$ and since $\nabla h (X_t^{\nust})$ is a process of finite variation, the optimality dynamics~\eqref{eq:Thm-Optimality-BSDE} imply that $[Z]_t = [ e^{-\gamma_t} \M ]_t] = e^{-\gamma_t} [ \M ]_t $ 

Inserting the quadratic co-variation bound into equation~\eqref{eq:proof-energy-dynamics-simplified-3} and using the super-martingale property, we obtain the final result
\begin{align*}
	\E\left[ \left( f(X_t) - f(\xst)\right)\right]
 	&\leq e^{-\beta_t}
 	\left(
 	C_0 + \frac{1}{2} \E\left[ \, [ \nabla h ( X + e^{-\alpha_t} \nu ), \nu ]_t\right]
 	\right)
 	\\&
 	\leq e^{-\beta_t}
 	\left(
 	C_0 + \frac{1}{2} e^{-2\gamma_t} \E\left[ \, [ \mcM ]_t\right]
 	\right)
 	\\ &\leq
 	(C_0 + \frac{1}{2}) e^{-\beta_t} \max\left\{ 1 \,,\, e^{-2\gamma_t} \E\left[ \, [ \mcM ]_t\right] \right\}
 	\\ &= O \left(
 	 	 	 e^{-\beta_t} \max\left\{ 1 \,,\, e^{-\beta_t + 2\gamma_t} \E\left[ \, [ \mcM ]_t\right] \right\} \right)
	\;,
\end{align*}
as desired.
}
\end{proof}

\section{Proofs of Propositions~\ref{prop:FOSP-Martingale-SGD-cts-Solution} and Proposition~\ref{prop:perturbation-solution-filter}}

Both of the proofs contained in this sections are applications of the momentum representation of the optimizer dynamics, and the FOSP approximation to the solution of the optimality FBSDE~\eqref{eq:Thm-Optimality-BSDE}.

\subsection{Proof of Proposition~\ref{prop:FOSP-Martingale-SGD-cts-Solution}}
\label{sec:proof-prop-FOSP-Martingale-SGD-cts-Solution}

\begin{proof}
	Using Proposition~\ref{prop:Perturbation-Solution}, we find that the solution to the FOSP takes the form 
	\begin{equation*}
		\tpz_t
		=
		\E\left[\left.
		\int_{t}^{T}
		e^{\gamma_t + \alpha_t + \beta_t}   
		\, \nabla f\left( X_u \right) \, du
		- e^{\delta_T} \nabla f(X_T^{\nust} )
		\right\lvert  \F_t \right]
		\;.
	\end{equation*}
	Applying Fubini's theorem, and the martingale property of $\E\left[ \nabla f\left( X_u \right) \lvert \F_u \right] = \nicefrac{g_u}{(1+\rho^2)}$, we find that
	\begin{align*}
		\tpz_t
		&=
		\E\left[\left.
		\int_{t}^{T}
		e^{\gamma_u + \alpha_u + \beta_u}   
		\, \nabla f\left( X_u \right) \, du
		- e^{\delta_T} \nabla f(X_T^{\nust} )
		\right\lvert  \F_t \right]
		\\ &=
		\int_{t}^{T}
		e^{\gamma_u + \alpha_u + \beta_u}   
		\, \E\left[ \nabla f(X_u^{\nust} ) \lvert \F_t \right] \, du
		- e^{\delta_T}  \E\left[ \nabla f(X_T^{\nust} ) \lvert \F_t \right]
		\\ &=
		\int_{t}^{T}
		e^{\gamma_u + \alpha_u + \beta_u}   
		\, \E\left[ \nabla f(X_t^{\nust} ) \lvert \F_t \right] \, du
		- e^{\delta_T}  \E\left[ \nabla f(X_t^{\nust} ) \lvert \F_t \right]
		\\ &=
		g_t (1+\rho^2)^{-1}
		\left(\int_{t}^{T}
		 		 		e^{\gamma_u + \alpha_u + \beta_u}   
		 		 		 \, du
		 		 		- e^{\delta_T}\right)
		 		 		\;.
	\end{align*}
	Inserting expression above into equation~\eqref{eq:nu-as-function-of-momentum}, and re-arranging terms, we obtain the desired result.
\end{proof}

\subsection{Proof of Proposition~\ref{prop:perturbation-solution-filter}}
\label{sec:proof-prop-perturbation-solution-filter}

\begin{proof}
		Using Proposition~\ref{prop:Perturbation-Solution}, we find that the solution to the FOSP takes the form 
	\begin{equation*}
		\tpz_t
		=
		\E\left[\left.
		\int_{t}^{T}
		e^{\gamma_t + \alpha_t + \beta_t}   
		\, \nabla f\left( X_u \right) \, du
		- e^{\delta_T} \nabla f(X_T^{\nust} )
		\right\lvert  \F_t \right]
		\;.
	\end{equation*}
	Applying Fubini's theorem, and noting that $\E[\, \nabla_i f(X_{t+h}) \lvert y_{i,t} ] = \sum_{j=1}^{\dt} ( b^{\T} e^{-A h} )_{j} \, y_{\cdot,j,t}$, we obtain
	\begin{align*}
		\tpz_t
		&=
		\E\left[\left.
		\int_{t}^{T}
		e^{\gamma_u + \alpha_u + \beta_u}   
		\, \nabla f\left( X_u \right) \, du
		- e^{\delta_T} \nabla f(X_T^{\nust} )
		\right\lvert  \F_t \right]
		\\ &=
		\int_{t}^{T}
		e^{\gamma_u + \alpha_u + \beta_u}   
		\, \E\left[ \nabla f(X_u^{\nust} ) \lvert \F_t \right] \, du
		- e^{\delta_T}  \E\left[ \nabla f(X_T^{\nust} ) \lvert \F_t \right]
		\\ &=
		\int_{t}^{T}
		e^{\gamma_u + \alpha_u + \beta_u}   
		\, 
		\left( \sum_{j=1}^{\dt} ( b^{\T} e^{-A (u-t)} )_{j} \, y_{\cdot,j,t} \right)\, du
		- e^{\delta_T} \left( \sum_{j=1}^{\dt} ( b^{\T} e^{-A (T-t)} )_{j} \, y_{\cdot,j,t} \right)
		\\ &=
		\sum_{j=1}^{\dt} 
		\left(
		\int_{t}^{T}
		e^{\gamma_u + \alpha_u + \beta_u}   
		\,
		\left( b^{\T} e^{-A (u-t)} \right)_{j} \, du
		-
		e^{\delta_T} ( b^{\T} e^{-A (T-t)} )_{j} 
		\right)  y_{\cdot,j,t}
	\end{align*}
	Inserting expression above into equation~\eqref{eq:nu-as-function-of-momentum}, and re-arranging terms, we obtain the desired result.
\end{proof}

\end{document}